\documentclass[5pt]{article}

\usepackage[letterpaper]{geometry}
\usepackage{spconf,amsmath,epsfig}
\usepackage{enumitem}

\usepackage{fancyhdr}

\usepackage{spconf,amsmath,epsfig}

\usepackage{subfigure}
\usepackage{float}

\usepackage{tabularx}
\usepackage{tabu}
\usepackage{booktabs}
\usepackage{amsthm}
\usepackage{amsmath}
\usepackage{amsfonts}
\usepackage{algorithm}
\usepackage{algorithmicx}
\usepackage{algpseudocode}
\usepackage{makecell,rotating,multirow,diagbox}

\newtheorem{proposition}{Proposition}

\begin{document}\sloppy

\def\x{{\mathbf x}}
\def\L{{\cal L}}

\title{TLR: TRANSFER LATENT REPRESENTATION FOR UNSUPERVISED \\ DOMAIN ADAPTATION}
%

%
%

\name{Pan Xiao$^1$, Bo Du*$^1$, Jia Wu$^2$, Lefei Zhang$^1$, Ruimin Hu$^1$ and Xuelong Li$^{3}$\thanks{*Corresponding author (email: remoteking@whu.edu.cn). Thanks to the exchange program of graduate students from Wuhan University, the National Natural Science Foundation of China (No. 61471274 and 61771349), the MQNS (No. 9201701203), the MQ Enterprise Partnership Scheme Poilt Res (No. 9201701455), and the major fund project (No. U1736206) for funding.}}
\address{
${}^1$School of Computer, Wuhan University, Wuhan 430072, Hubei, China\\
${}^2$Department of Computing, Macquarie University, Sydney, NSW 2109, Australia\\
${}^3$Xi'an Institute of Optics and Precision Mechanics,\\
Chinese Academy of Sciences, Xi'an 710119, Shaanxi, China\\
}


\maketitle

\hyphenation{web-cam}\hyphenation{latent}\hyphenation{methods}\hyphenation{data-sets}\hyphenation{original}\hyphenation{sub-optimal}\hyphenation{classifier}
\hyphenation{diagonal}\hyphenation{cross-domain}\hyphenation{aims}\hyphenation{auto-encoder}\hyphenation{developed}\hyphenation{causes}\hyphenation{space}
\hyphenation{do-mains}\hyphenation{unsatisfactory}\hyphenation{straight-forward}\hyphenation{also}\hyphenation{represen-tations}\hyphenation{prediction}
\hyphenation{labeled}\hyphenation{impossible}\hyphenation{analyze}\hyphenation{maximization}\hyphenation{pro-perties}
\begin{abstract}
Domain adaptation refers to the process of learning prediction models in a target domain by making use of data from a source domain.
Many classic methods solve the domain adaptation problem by establishing a common latent space, which may cause the loss of many important properties across both domains. 
In this manuscript, we develop a novel method, \emph{transfer latent
representation} (TLR), to learn a better latent space.
Specifically, we design an objective function based on a simple linear autoencoder to derive the latent representations of both domains.
The encoder in the autoencoder aims to project the data of both domains into a robust latent space.
Besides, the decoder imposes an additional constraint to reconstruct the original data, which can preserve the common properties of both domains and reduce the noise that causes domain shift.
Experiments on cross-domain tasks demonstrate the advantages of TLR over competing methods.
\end{abstract}


\section{Introduction}
\hyphenation{model}\hyphenation{nearly}\hyphenation{enough}\hyphenation{unlabeled}


Recently, online images and videos grow exponentially, which has created a strong demand for technologies to analyze the multimedia content.
Unfortunately, labels for these new visual images are in short supply and it is nearly impossible to learn a good visual category model without enough labels. In real-world applications, there exist many labeled datasets in some old domains. Can we use these labeled datasets (i.e. the source domain) to handle unlabeled datasets (the target domain)?
To answer this question, a technique named domain adaptation (DA) has been developed.

DA is very important when the labels for target domain data are lacking \cite{Yuan2012Information}.
For example, we can obtain some labeled images drawn from the Internet (i.e. the source domain) and some unlabeled images captured by cameras (the target domain), and that both domains contain the same objects. It is believed that a model trained in the source domain can significantly improve classification accuracy in the target domain after the common properties of both domains are extracted \cite{Pan2010A}.
Taking action recognition by surveillance cameras in Fig.1 as an example, we have a series of action \hyphenation{shots}shots captured from two different angles.
If we were to directly use a set of labeled images on the left to classify unlabeled pictures on the right, the classification accuracy might be unsatisfactory.
However, considering that both \hyphenation{sets}sets of pictures contain the same set of actions,
we believe a better prediction result can be obtained by recognizing and utilizing the commonalities between the image sets in classification.

\begin{figure}
\begin{minipage}{0.98\linewidth}
  \centerline{\includegraphics[width=6cm]{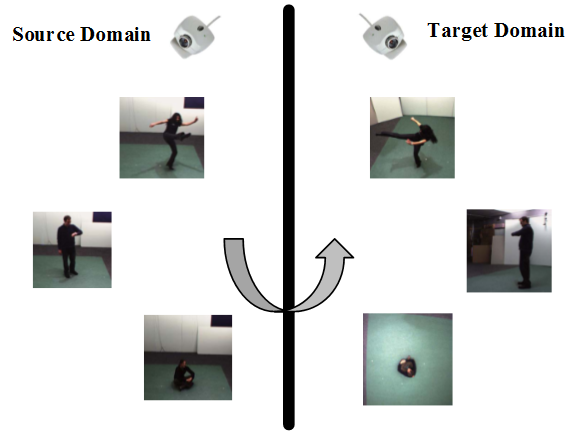}}
\end{minipage}
\caption{Cross-camera action recognition.}
\label{fig:res}
\end{figure}

Domain adaptation methods can be divided into two categories: semi-supervised DA and unsupervised DA according to the availability of labeled instances in the target domain.  In this work, we focus on the unsupervised scenario, which is hard
to solve since the labels for the target domain are totally non-existent.
Many well-known methods have been proposed to solve the unsupervised domain adaptation problem. One straightforward solution is to project both domains into a common latent space.
For example, Fernando \emph{et al.} \cite{Fernando2013Unsupervised} proposed to learn a linear projection aligning the source and target domains.
Gong \emph{et al.} \cite{Gong2012Geodesic} claimed that new latent representations could be obtained by regarding the subspaces of both domains as points in Grassmann manifolds.
Pan \emph{et al.} \cite{Pan2011Domain} and Yan \emph{et al.} \cite{Yan2016Domain_Y} projected the source and target domain data into a  Reproducing Kernel Hilbert Space (RKHS) to obtain the latent representations of both domains. Although all these state-of-the-art methods have achieved promising results, there is still room for improvement, mainly because those methods may result in the loss of many important properties of both domains that are helpful for model building when projection is performed.

In this paper, we propose a new method called Transfer Latent Representation (TLR) to learn a better latent space.
Specifically, we first follow the procedure outlined in \cite{Pan2011Domain} to obtain linearly separable source and target domain data by projecting both domains into an RKHS.
To avoid the loss of useful properties, we then design an objective function based on a linear autoencoder to derive the latent representations of both domains.
The encoder of the autoencoder is set up to project both domains into a latent space, in the same way as the existing domain adaptation methods.
Besides, the decoder exerts an additional constraint, that is, the original data must be reconstructed by the projection.
It is supposed that the use of this additional reconstruction constraint can assist in preserving the common properties of both domains and reducing the noise that causes domain shift.
The Maximum Mean Discrepancy (MMD) \cite{gretton2007kernel} between the latent representations is also integrated into the objective function, so that the function is able to further narrow the distance between different domain distributions.
Finally, we obtain the latent representations of both domains in a latent space.

\section{Preliminaries}\hyphenation{domain}
In this work, we aim to solve the unsupervised domain adaptation problem:
how to best label the unlabeled target domain data in an unsupervised manner by training a model on labeled data in a relevant source domain.
Firstly, we denote $X_{S}= \left \{ x_{S_{1}},...,x_{S_{n_{1}}} \right \}\in\mathbb{R}^{d\times n_{1}}$ as the source domain data and $X_{T}= \left \{ x_{T_{1}},...,x_{T_{n_{2}}} \right \}\in\mathbb{R}^{d\times n_{2}}$ as the target domain data.
Here, $d$ is the dimension of each instance, while $n_{1}$ and $n_{2}$ are the number of samples in the source and target domains respectively.
The source domain data labels are denoted as $Y_{S}= \left \{ y_{S_{1}},...,y_{S_{n_{1}}} \right \}\in\mathbb{R}^{n_{1}}$, where $y_{S_{i}}$ is the label of the corresponding source domain sample $x_{S_{i}}$.
Similarly, the predicted labels of the target domain are denoted as $\hat{Y}_{T}= \left \{ y_{T_{1}},...,y_{T_{n_{2}}} \right \}\in\mathbb{R}^{n_{2}}$.
Our goal is to train a classifier based on $X_{S}$ and $Y_{S}$, then predict the labels of the target domain data as accurately as possible.

\subsection{Maximum Mean Discrepancy}
\emph{Maximum Mean Discrepancy} (MMD) has been successfully used to solve the domain adaptation problem \cite{Pan2008Transfer,Long2013transfer}.
By computing on $X_{S}$ and $X_{T}$, a non-parametric distance estimate between domain distributions can be directly obtained. Here, let
\begin{equation}
\label{eqn_example}
MMD(X_{S},X_{T})=\left \| \frac{1}{n_{1} }\sum_{i=1}^{n_{1}}f(x_{S_{i}}) - \frac{1}{n_{2} }\sum_{i=1}^{n_{2}}f (x_{T_{i}}) \right \|_{\mathcal{H}}^{2}
\end{equation}
where $\mathcal{H}$ is a universal Reproducing Kernel Hilbert Space (RKHS), and $f : \mathcal{X} \rightarrow  \mathcal{H}$ denotes the non-linear transformation. By means of the kernel trick, (i.e., $k(x_{i} ,x_{j})=f(x_{i})f(x_{j})'$), we can rewrite (1) as

\begin{equation}
\label{eqn_example}
MMD(X_{S},X_{T})=tr(KL).
\end{equation}
in which
\begin{equation}
\label{eqn_example}
K
= \begin{bmatrix}
K_{S,S} & K_{S,T}\\
K_{T,S} & K_{T,T}
\end{bmatrix}
=\begin{bmatrix}
H_{S}\\
H_{T}
\end{bmatrix}\in \mathbb{R}^{(n_{1} + n_{2}) \times (n_{1} + n_{2})}
\end{equation}
is a symmetric kernel matrix. The elements of $K_{S,S}$, $K_{T,T}$, $K_{T,S}$ and $K_{S,T}$ are the values of $k(x_{i}, x_{j})$ when ($x_{i}$,$x_{j}$) belongs to the source domain, target domain, and two cross domains respectively.
$H_{S}$ and $H_{T}$ denote the source and target domain samples mapped into the RKHS.
$L$ is the MMD matrix and can be described as follows:
\begin{eqnarray}
\label{eqn_example}
(L)_{i,j}=\left\{\begin{matrix}
\frac{1}{n_{1}n_{1}}, & x_{i},x_{j}\in X_{S} \\
\frac{1}{n_{2}n_{2}}, & x_{i},x_{j}\in X_{T} \\
\frac{-1}{n_{1}n_{2}}, &otherwise
\end{matrix}\right.
\end{eqnarray}

\section{Transfer Latent Representation}
The proposed model consists of two main stages. First,
we map the data of both domains into the RKHS and obtain $H_{S}$ and $H_{T}$ according to (3).
We then derive a matrix $W$ based on the simple linear autoencoder, which projects $H_{S}$ and $H_{T}$ into a latent space.

\subsection{Simple Linear Autoencoder}
The simplest form of an autoencoder is linear. Here, there is no activation function in the single hidden layer.
The encoder is used to project the input data into the single hidden layer, while the decoder projects it back to the original feature space.
Suppose that $X \in \mathbb{R}^{n\times (n_{1} + n_{2})}$ is the input data matrix with $n$ samples.
We want to obtain a projection matrix $W \in \mathbb{R}^{(n_{1} + n_{2}) \times k}$ in order to explore the $k$-dimensional latent representation $P \in \mathbb{R}^{n \times k}$.
The obtained $P$ is projected back to the original feature space by means of transpose matrix of $W$ so that it becomes $\hat{X} \in \mathbb{R}^{n\times (n_{1} + n_{2})}$.
Note here that $k$ is smaller than $(n_{1} + n_{2})$.
By minimizing the reconstruction error,
we have
\begin{equation}
\label{eqn_example}
\min_{W,W^\top} \left \| PW^\top-X \right \|_{F}^{2} \quad
s.t. \quad P = XW.
\end{equation}
The above objective makes $X$ and $\hat{X}$ as similar as possible.

\subsection{Model Formulation}
We apply the simple linear autoencoder to the source domain in RKHS ($H_{S}$), and the target domain in RKHS ($H_{T}$).
One significant advantage of the autoencoder is that it can reconstruct the input features of the source and target domains, which forces the latent representations of both domains to maintain as many important properties as possible.
Formally, we have
\begin{align}
\label{eqn_example}
\min_{W} \left \| P_{S}W^\top -H_{S} \right \|_{F}^{2}
\quad s.t. \quad P_{S} = H_{S}W.
\end{align}
\begin{align}
\label{eqn_example}
\min_{W} \left \| P_{T}W^\top -H_{T} \right \|_{F}^{2}
\quad s.t. \quad P_{T} = H_{T}W.
\end{align}
where $P_{S}$ and $P_{T}$ are the latent representations for the source and target domains respectively. 

To further narrow the distance between the distributions of both domains, we minimize the MMD of the two latent representations ($P_{S}$ and $P_{T}$).
More specifically, we have,
\begin{align}
\label{eqn_example}
\min_{W} MMD(P_{S},P_{T}).
\end{align}

By combining (6) and (7) with (8), the proposed model can be summarized as follows:
\begin{align}
\min_{W}F(W)&= MMD(P_{S},P_{T})+\alpha\left \| P_{S}W^\top -H_{S} \right \|_{F}^{2}\nonumber\\
        &+\beta\left \| P_{T}W^\top -H_{T} \right \|_{F}^{2}.
\end{align}
where $\alpha$ and $\beta$ are trade-off parameters.

At this point, our goal is to obtain the optimal $W$ that will minimize the objective function (9). The derived $W$ can project $H_{S}$ and $H_{T}$ into a common latent space.
It is believed that, in the latent space, the noise that causes domain shift will be reduced and the common properties of different domains will be extracted.
\subsection{Optimization}
In this section, we introduce three propositions in turn. An efficient optimization algorithm is then designed.
\begin{proposition}
The term (8) can be rewritten as
\begin{align}
\label{eqn_example}
\min_{W} tr(W^\top KLKW).
\end{align}
\end{proposition}

\begin{proof}
According to (2), we have
\begin{align}
\label{eqn_example}
MMD(P_{S},P_{T})
=tr(K_{\mathcal{H}}L).
\end{align}
where
\begin{align}
\label{eqn_example}
K_{\mathcal{H}}
&=\begin{bmatrix}
H_{S}WW^\top H_{S}^\top & H_{S}WW^\top H_{T}^\top\\
H_{T}WW^\top H_{S}^\top & H_{T}WW^\top H_{T}^\top
\end{bmatrix}\nonumber\\
&=\begin{bmatrix}
H_{S}\\
H_{T}
\end{bmatrix}WW^\top
\begin{bmatrix}
H_{S}^\top & H_{T}^\top
\end{bmatrix}\nonumber\\
&=KWW^\top K^\top.
\end{align}
It is worth noting that we use $k(x_{i} ,x_{j})=x_{i}x_{j}'$ directly, since the source and target domain data have been mapped into the RKHS.

By substituting equation (12) into (11),  the MMD of both domains in the latent space can be described as
\begin{align}
MMD(P_{S},P_{T})
&=tr(KWW^\top K^\top L)\nonumber\\
&=tr(W^\top K^\top LKW).
\end{align}
Since $K$ is a symmetric matrix, we have
\begin{align}
MMD(P_{S},P_{T})
=tr(W^\top KLKW).
\end{align}
Finally, we obtain an equivalent problem (10).
\end{proof}
By combining (14) with (9), we can summarize our
objective function as
\begin{align}
\min_{W}F(W)& = tr(W^\top KLKW)+\alpha\left \| H_{S}WW^\top -H_{S} \right \|_{F}^{2}\nonumber\\
        &+\beta\left \| H_{T}WW^\top -H_{T} \right \|_{F}^{2}.
\end{align}

\begin{proposition}
The objective function (15) can be rewritten more compactly as
\begin{align}
\label{eqn_example}
\min_{W} tr(WW^\top AWW^\top  + W^\top BW-2W^\top AW+A).
\end{align}
where
\begin{align}
\label{eqn_example}
A = KMK,B=KLK.
\end{align}
and
\begin{align}
\label{eqn_example}
M=\begin{bmatrix}
\alpha I_{n_{1}\times n_{1}} & 0_{n_{1}\times n_{2}}\\
0_{n_{2}\times n_{1}} & \beta I_{n_{2}\times n_{2}}
\end{bmatrix}.
\end{align}
\end{proposition}
The proof is given in the \textbf{Appendix}.

Here, it is evident that the solution may collapse to one point ($W = 0$).
To avoid such an occurrence, we impose a constraint $W^\top AW=I$ into (16).
Accordingly, the optimization problem with constraint can be summarized as follows:
\begin{eqnarray}
\label{eqn_example}
&\min_{W} &tr(W^\top W + W^\top BW)\nonumber\\
&s.t. &W^\top AW=I.
\end{eqnarray}
We can solve problem (19) efficiently using the Lagrangian multiplier method, so that it eventually becomes the following optimization problem:
\begin{proposition}
Problem (19) can be rewritten as
\begin{align}
\label{eqn_example}
\max_{W}tr((W^\top (I+B)W)^{-1}W^\top AW).
\end{align}
\end{proposition}

\begin{proof}
The Lagrangian function of (19) is
\begin{align}
\label{eqn_example}
L(W,Z) = tr(W^\top (I+B)W)-tr((W^\top AW-I)Z).
\end{align}
where $Z$ is a matrix with the Lagrange multipliers in the diagonal. Setting $\frac{\partial L(W,Z) }{\partial W}=0$, we have
\begin{align}
\label{eqn_example}
(I+B)W = AWZ.
\end{align}
To simplify the Lagrangian function $L(W,Z)$,
we multiply both sides of equation (22) on the left by $W^\top$ and combine it with (21),
so that we have
\begin{align}
\label{eqn_example}
\min_{W}tr((W^\top AW)^{-1}W^\top (I+B)W).
\end{align}
As the matrix $I + B$ is non-singular, an equivalent trace maximization problem (20) can be obtained.
\end{proof}

The solution of $W$ in (20) is the eigenvectors corresponding to the $k$ leading eigenvalues of $(I+B)^{-1}A$.
The whole procedure of TLR is summarized in Algorithm 1.
\begin{algorithm}
\caption{Transfer Latent Representation (TLR)}\label{DASLFN}
\begin{algorithmic}[1]
\renewcommand{\algorithmicrequire}{\textbf{Input:}}
\renewcommand{\algorithmicensure}{\textbf{Output:}}
\Require Labeled source domain data $X_{S}$, unlabeled target domain data $X_{T}$, source labels $Y_{S}$, latent space dimension $k$, trade-off parameters $\alpha$ and $\beta$;
\Ensure Predicted target labels $\hat{Y}_{T}$;
\State Compute matrices $K$, $H_{S}$, and $H_{T}$ according to (3);
\State Compute matrices $L$ and $M$ using (4) and (18) respectively;
\State Compute matrices $A$ and $B$ according to (17);
\State Obtain projection matrix $W$ according to the $k$ leading eigenvalues of $(I+B)^{-1}A$;
\State $P_{S} = H_{S}W$;
\State $P_{T} = H_{T}W$;
\State $\hat{Y}_{T}$ $\leftarrow$ \emph{Classifier}($P_{S}$, $P_{T}$, $Y_{S}$ );
\end{algorithmic}
\end{algorithm}

\hyphenation{invited}\hyphenation{Sequences}\hyphenation{metric}\hyphenation{linear}\hyphenation{Finally}\hyphenation{manifolds}\hyphenation{sub-spaces}
\hyphenation{res-pective}\hyphenation{constructs}\hyphenation{parameters}\hyphenation{values}\hyphenation{only}\hyphenation{classifi-cation}\hyphenation{datasets}
\hyphenation{validation}\hyphenation{vec-tor}\hyphenation{SURF}\hyphenation{randomly}\hyphenation{into}\hyphenation{features}\hyphenation{experiments}
\section{Experiments}
To demonstrate the efficacy of our proposed TLR approach, we perform experiments on two cross-domain datasets: 1) the 4DA dataset, and 2) the IXMAS dataset.
\subsection{Data Preparation}
\textbf{4DA Dataset:}
We adopt the public 4DA dataset \cite{Saenko2010Adapting}, which contains four domains, namely \textbf{Amazon}, \textbf{Webcam}, \textbf{DSLR}, and \textbf{Caltech-256}.
Fig.2 shows the MONITOR examples from the four domains.
The differences between them are obvious. For example, the monitor screens from Amazon and Caltech-256 display colorful images while the monitor screens from Webcam and DSLR are black.
For the experiments, we follow the procedure in \cite{Saenko2010Adapting} to extract SURF features.
Each image corresponds to an 800-dimensional vector.
The instances are then standardized by z-score.
We randomly select two \hyphenation{diff-er-ent}different domains from A (Amazon), W (Webcam), D (DSLR), and C (Caltech-256).
Thus there are a total of $4\times 3 = 12$ cross-domain pairs, e.g., $A\rightarrow W$, $A\rightarrow D$,$A\rightarrow C$,$\dots$,$C\rightarrow D$.\newline
\textbf{IXMAS Dataset:}
The Inria Xmas Motion Acquisition Sequences (IXMAS)\footnote{http://4drepository.inrialpes.fr/public/viewgroup/6} is a multi-view action recognition dataset containing 11 actions.
Each action is regarded as a category.
As can be seen in Fig. 2, five cameras (cam0, cam1, ..., cam4) are used to capture the actions from different perspectives.
Each perspective represents a domain.
Thus, five domains are included in this dataset.
Twelve actors are invited to perform each action three times, giving $12\times 3 = 36$ instances per class.
The feature extraction is based on the settings in \cite{Liu2011Cross}.
We conduct experiments on $5$ cross-domain pairs ($c0\rightarrow c1$, $c1\rightarrow c2$, $\dots$, $c4\rightarrow c0$).
\begin{figure}[!t]
\centering
\includegraphics[width=3.4in]{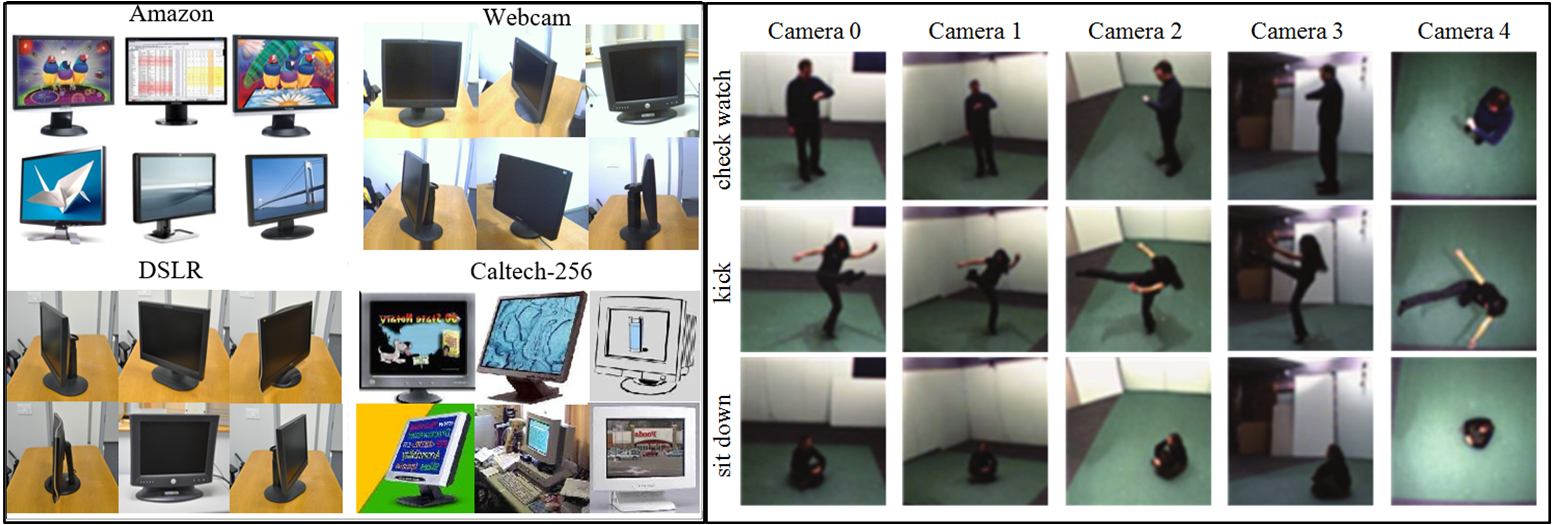}
\caption{Example images from the 4DA and IXMAS datasets.}
\label{3DA}
\end{figure}
\subsection{Comparison methods}
To evaluate the robustness of the proposed TLR approach, we compare TLR with six competitive methods:
Principal Component Analysis (\textbf{PCA}), Information-Theoretical Learning (\textbf{ITL} \cite{Yuan2012Information}), Subspace Alignment (\textbf{SA} \cite{Fernando2013Unsupervised}), Transfer Component Analysis (\textbf{TCA} \cite{Pan2011Domain}), Geodesic Flow Kernel (\textbf{GFK} \cite{Gong2012Geodesic}), Maximum Independence Domain Adaptation (\textbf{MIDA} \cite{Yan2016Domain_Y}).
Following the settings in \cite{Gong2012Geodesic}, 1-NN is chosen as the base classifier, where the source and target domains are regarded as the training set and test set respectively.
This is so that we do not need to tune cross-validation parameters when training a model.
We first compare our method with PCA, where both domains are mapped into their respective subspaces.
In particular, ITL \cite{Yuan2012Information}, SA \cite{Fernando2013Unsupervised}, TCA \cite{Pan2011Domain}, GFK \cite{Gong2012Geodesic}, and MIDA \cite{Yan2016Domain_Y} obtain a \hyphenation{domain-invariant}domain-invariant feature subspace in different ways.
ITL optimizes an information-theoretic metric and learns the feature space discriminatively.
SA learns a linear projection that \hyphenation{aligns}aligns both domains using subspace alignment,
while TCA maps data from the source and target domains into an RKHS in order to transfer components across domains.
GFK \hyphenation{extracts} extracts an infinite number of subspaces and constructs geodesic flows between them.
Here, the subspaces of both domains are regarded as points in Grassmann manifolds.
Finally, MIDA maximizes the independence of the derived and the instance features in order to reduce the difference between domains.
\subsection{Implementation Details}

In TLR, we need to tune two model parameters: the trade-off parameters $\alpha$ and $\beta$.
Since the distributions of both domains are different, obtaining the optimal parameters by cross validation is impossible.
We thus evaluate TLR by designing a search space according to \cite{Long2013transfer}.
The search range for $\alpha$ and $\beta$ is $\left \{10^{-5},10^{-4},10^{-3},10^{-2},10^{-1},10^{0} \right \}$ separately.
Similarly, the optimal dimension of latent space $k$ is obtained by searching $\left \{10,20,...,200 \right \}$.
The best results are then reported.
For the other comparison methods mentioned above,
we tune the parameters according to the original paper and report their best performance.
\begin{table}[htbp]
\footnotesize
  \centering
  \caption{Classification accuracy (\%) for all methods on the 4DA and IXMAS datasets.}
    \begin{tabular}{|c|c|c|c|c|c|c|c|c|c|}
    \hline
    Methods    & PCA   & SA   & ITL    & TCA   & GFK   & MIDA  & TLR \\
    \hline
    \hline
    C$\rightarrow $A     & 32.5   & 31.4  & 35.4 & 37.6  & 35.8  & 37.3  & \textbf{38.7} \\
    \hline
    C$\rightarrow $D     & 26.3   & 33.0  & 33.1  & 35.0  & 35.7  & 35.3  & \textbf{38.1} \\
    \hline
    C$\rightarrow $W     & 25.1    & 26.9  & 28.0 & 31.9  & 31.0  & 31.8  & \textbf{34.4} \\
    \hline
    A$\rightarrow $C     & 31.9    & 32.2  & 34.9 & 34.9  & 33.8  & 34.6  & \textbf{35.2} \\
    \hline
    A$\rightarrow $D     & 25.7   & 28.4 & 30.0  & 30.1  & 33.2  & 29.7  & \textbf{34.8} \\
    \hline
    A$\rightarrow $W     & 28.7    & 28.8 & 29.6 & 32.6  & 33.0  & 32.6  & \textbf{35.1} \\
    \hline
    D$\rightarrow $C     & 27.7    & 30.8 & 31.6 & 31.1  & 27.8  & 30.7  & \textbf{32.2} \\
    \hline
    D$\rightarrow $A     & 31.0    & 31.8 & 34.1 & 34.2  & 31.5  & 33.9  & \textbf{35.1} \\
    \hline
    D$\rightarrow $W     & 59.5   & 78.7 & \textbf{78.8} & 75.3  & 69.1  & 75.5  & \textbf{78.8} \\
    \hline
    W$\rightarrow $C     & 25.5    & 25.1 & 27.8 & 29.7  & 28.9  & 29.7  & \textbf{30.6} \\
    \hline
    W$\rightarrow $A     & 31.2    & 30.0 & 32.4 & 30.3  & \textbf{33.7} & 31.4  & 30.0 \\
    \hline
    W$\rightarrow $D    & 68.1    & 81.1 & 82.4 & 80.1  & 78.2  & 81.5  & \textbf{86.3} \\
    \hline
    Average   & 34.4    & 38.2 & 39.8 & 40.2  & 39.3  & 40.3  & \textbf{42.4} \\
    \hline
    \hline
    c0$\rightarrow $c1    & 8.6  & 14.7 & 15.4 & 30.0  & 14.2  & 25.1  & \textbf{35.2} \\
    \hline
    c1$\rightarrow $c2     & 9.4    & 13.6 & 20.8 & 18.8  & 15.5  & 19.8  & \textbf{21.1} \\
    \hline
    c2$\rightarrow $c3     & 10.5    & 9.0 & 13.6 & 12.3  & 8.9 & 12.4  & \textbf{19.9} \\
    \hline
    c3$\rightarrow $c4    & 8.2    & 14.3 & 17.2 & 21.5  & 17.8  & 20.9  & \textbf{23.9} \\
    \hline
    c4$\rightarrow $c0    & 13.5    & 14.2 & 16.5 & 24.6  & 16.3  & 24.3  & \textbf{27.5} \\
    \hline
    Average  & 10.0    & 13.2 & 16.7 & 21.4  & 14.5  & 20.5  & \textbf{25.5} \\
    \hline
    \end{tabular}%
  \label{3DAtable}\vspace{-0.2cm}
\end{table}%

We follow the settings in \cite{Gong2012Geodesic} to select the training set and test set when conducting experiments on 4DA dataset.
For the IXMAS dataset, we randomly select 30 labeled source domain instances per category as the training set and treat all target domain samples as the test set.
We run experiments ten times at random for the 17 cross-domain image (object and action) pairs and the average classification accuracy is then reported in Table 1.
\subsection{Experimental Results}
\hyphenation{dimen-sion}
The best result for each cross-domain pair is shown in bold.
We observe that TLR outperforms all classic unsupervised domain adaptation \hyphenation{me-thods}methods.
TLR's average classification accuracies on the 4DA and \hyphenation{IXMAS}IXMAS datasets are 42.4\% and 25.5\% respectively,
and the performance is improved by 2.1\% and 4.1\% relative to the best comparison method.
This demonstrates that TLR can obtain more robust latent representations than its competitors when facing cross-domain recognition tasks.

Secondly, out of all methods studied, the classification results of PCA are the worst.
This is because PCA is not designed to solve domain adaptation problem.
SA's performance is slightly better than that of PCA, since SA adapts both domains in PCA subspaces, which further improves the classification accuracy.

Thirdly, TLR significantly outperforms ITL. The classification accuracy of ITL on IXMAS is not very good.
A major limitation of ITL is that it assumes that data from the source domain and target domains are tightly clustered.
This assumption may be invalid on many datasets.

Note that TCA, somewhat like TLR, also learns latent representations using MMD.
However, the proposed method maintains more common properties between domains,
as the input features can be reconstructed by means of the simple linear autoencoder.

The performance of GFK on the 4DA dataset is good, but poor on the IXMAS dataset.
In GFK, the latent space dimension should be small enough to guarantee that subspaces transit smoothly along the geodesic flow.
However, this may result in the loss of some important properties.
TLR, on the other hand, can obtain a more accurate latent space.


Lastly, MIDA achieves better performance than the other compared algorithms.
Theoretically, MIDA can learn features containing maximal independence with the domain features.
However, one possible drawback of MIDA is that it retains fewer common properties than TLR does.

\begin{figure}[H]
\centering
\subfigure[C$\rightarrow $A]{\includegraphics[width=2.6cm]{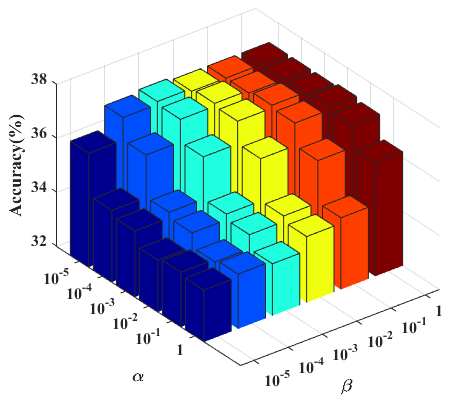}}
\subfigure[A$\rightarrow $D]{\includegraphics[width=2.6cm]{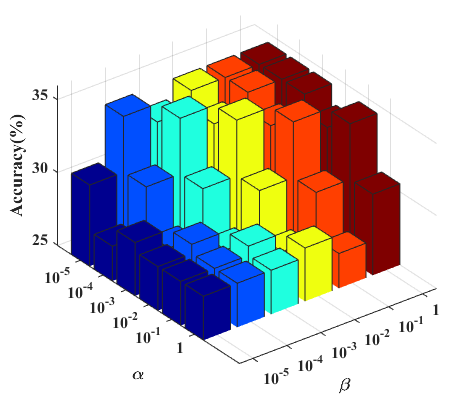}}
\subfigure[D$\rightarrow $W]{\includegraphics[width=2.6cm]{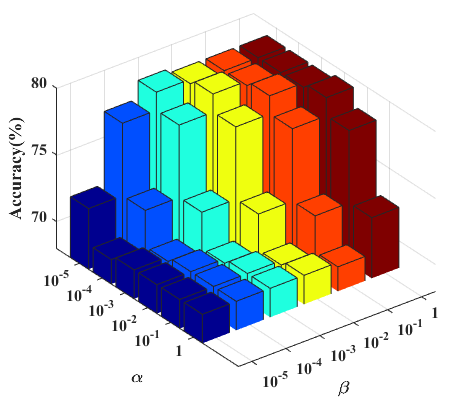}}
\subfigure[c0$\rightarrow $c2]{\includegraphics[width=2.6cm]{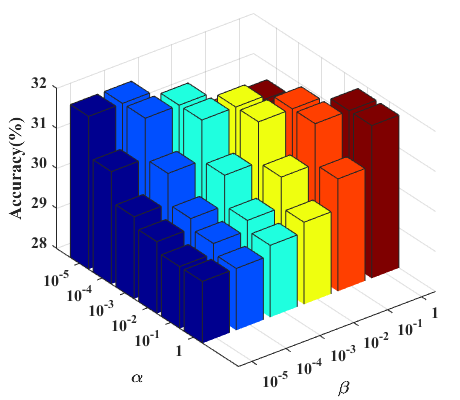}}
\subfigure[c1$\rightarrow $c3]{\includegraphics[width=2.6cm]{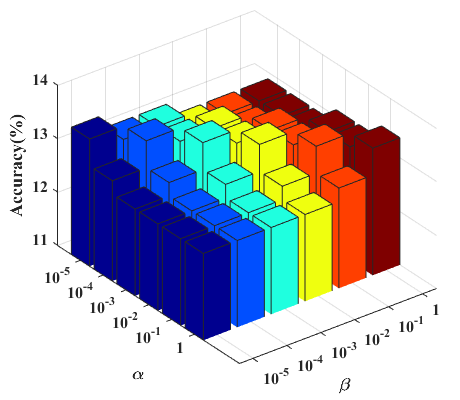}}
\subfigure[c4$\rightarrow $c1]{\includegraphics[width=2.6cm]{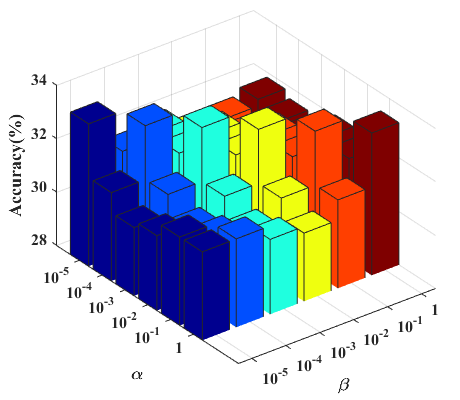}}
\caption{Parameter sensitivity analysis for TLR.}
\label{fig:res}\vspace{-0.2cm}
\end{figure}
\subsection{Parameter Sensitivity Analysis}
In the proposed TLR method, we need to tune two key parameters $\alpha$ and $\beta$.
$\alpha$ represents how much we weight the source domain data,
and $\beta$ denotes how much we weight the target domain data.
To evaluate the \hyphenation{ef-fect}effect of $\alpha$ and $\beta$ on the experimental results,
we run the experiments on the 4DA dataset (three tasks: C$\rightarrow$A, A$\rightarrow$D, and D$\rightarrow$W)
and IXMAS dataset (another three tasks: c0$\rightarrow$c2, c1$\rightarrow$c3, and c4$\rightarrow$c1) with different parameter values.
The two parameters $\alpha$ and $\beta$ are tuned from $\left \{10^{-5},10^{-4},10^{-3},10^{-2},10^{-1},1 \right \}$ separately.
For the 4DA dataset, the results in Fig.3(a,b,c) show that the performance of our model using $\alpha$ with small values and $\beta$ with large values is often better than other settings.
Moreover, the classification accuracy decreases greatly when $\alpha$ becomes larger and $\beta$ becomes smaller.
For the IXMAS dataset,
the results in Fig.3(d,e,f) show that a better classification accuracy can be obtained when the values of $\alpha$ and $\beta$ are the same, and
the classification accuracy will decrease largely with the increasing of the difference between $\alpha$ and $\beta$.
\section{Conclusion}
This paper proposes an unsupervised domain adaptation method called Transfer Latent Representation (TLR).
TLR aims to learn latent representations of the source and target domains.
In latent space, the common properties of both domains are preserved and noise that causes domain shift is reduced.
Experimental results on real-world cross-domain datasets demonstrate the effectiveness of our method.

In the future, we plan to extend TLR to solve the semi-supervised domain adaptation problem,
in which there are only a few data labels in the target domain.


\section{Appendix}
\subsection{Proof of Proposition 2}
The objective function $F(W)$ can be rewritten as
\begin{align*}
\label{eqn_example}
F(W)
&=tr(W^\top KLKW)\nonumber\\
&+\alpha tr((WW^\top H_{S}^\top-H_{S}^\top)(H_{S}WW^\top-H_{S}))\nonumber\\
&+\beta tr((WW^\top H_{T}^\top -H_{T}^\top)(H_{T}WW^\top-H_{T}))\nonumber\\
%
&=tr(W^\top KLKW)\nonumber\\
&+tr(WW^\top(\alpha H_{S}^\top H_{S}+\beta H_{T}^\top H_{T})WW^\top)\nonumber\\
&-2tr(W^\top(\alpha H_{S}^\top H_{S}+\beta H_{T}^\top H_{T})W)\nonumber\\
&
+tr(\alpha H_{S}^\top H_{S}
+\beta H_{T}^\top H_{T}).
\end{align*}
We let $A = \alpha H_{S}^\top H_{S} +\beta H_{T}^\top H_{T}$, so that $A$ can be described as
\begin{align}
A = \begin{bmatrix}
H_{S}^\top
H_{T}^\top
\end{bmatrix}\begin{bmatrix}
\alpha I_{n_{1}\times n_{1}} & 0_{n_{1}\times n_{2}}\\
0_{n_{2}\times n_{1}} & \beta I_{n_{2}\times n_{2}}
\end{bmatrix}\begin{bmatrix}
H_{S}\\
H_{T}
\end{bmatrix}=KMK.
\end{align}
in which $M = \begin{bmatrix}
\alpha I_{n_{1}\times n_{1}} & 0_{n_{1}\times n_{2}}\\
0_{n_{2}\times n_{1}} & \beta I_{n_{2}\times n_{2}}
\end{bmatrix}$.
The objective function $F(W)$ then can be compactly rewritten as
\begin{align*}
F(W)
&=tr(W^\top KLKW)
+tr(WW^\top AWW^\top)\nonumber\\
&-2tr(W^\top AW)
+tr(A).
\end{align*}
Futhermore, we let
\begin{align}
B = KLK.
\end{align}
By substituting (25) into $F(W)$, the proposed
model can be summarized as follows:
\begin{align}
\label{eqn_example}
F(W)
&=tr(W^\top BW)
+tr(WW^\top AWW^\top)\nonumber\\
&-2tr(W^\top AW)
+tr(A)\nonumber\\
&=tr(WW^\top AWW^\top  + W^\top BW-2W^\top AW+A)
\end{align}
Then, Proposition 2 is proven.

\ninept
\bibliographystyle{IEEEbib}
\bibliography{myicme}

\end{document}